\documentclass[fleqn,12pt,a4paper]{paper}

\usepackage{amssymb,amsmath,amsthm}
\usepackage{mathtools}
\usepackage{enumitem}
\usepackage{nicefrac}

\renewcommand\vec{\boldsymbol}

\newtheorem{theorem}{T{\sc heorem}}[section]
\newtheorem{corollary}{C{\sc orollary}}[section]
\newtheorem{lemma}{L{\sc emma}}[section]
\newtheorem{exercise}{E{\sc xercise}}[section]

\usepackage{hyperref}

\hypersetup{
  colorlinks  = true,     
  urlcolor    = darkgray, 
  linkcolor   = blue,     
  citecolor   = red       
}

\renewcommand{\emph}[1]{\textbf{#1}}
\newcommand{\pow}[1]{2^{#1}}
\newcommand{\partto}{\mathrel{\rightharpoonup}}
\newcommand{\abs}[1]{\left|#1\right|}
\newcommand{\norm}[2]{\left\lVert#1\right\rVert_{#2}}
\newcommand{\dtv}[2]{d_{\mathrm{TV}}\left(#1,#2\right)}
\newcommand{\drat}[2]{d_{\mathrm{rat}}\left(#1,#2\right)}
\DeclareMathOperator*{\Ex}{\mathrm{E}}
\newcommand{\pmin}{p_{\min}}
\DeclareMathOperator*{\argmax}{argmax}

\begin{document}

\author{Guillermo A. P\'erez\\
Universiteit Antwerpen}
\title{Partially Known MDPs}

\maketitle

\begin{abstract}
    In these notes we will tackle the problem of finding optimal
    policies for Markov decision processes (MDPs) which are not
    fully known to us. Our intention is to slowly transition from
    an offline setting to an online (learning) setting. Namely,
    we are moving towards \emph{reinforcement learning}.
\end{abstract}

\tableofcontents
\newpage

\section{Preliminaries}
As a reminder, a (stationary) MDP $\mathcal{M}$ is a $4$-tuple
$(S,A,P,r)$ where:
\begin{itemize}
    \item $S$ is a finite set of states,
    \item $A$ is a finite set of actions. In an abuse of notation, we also write $A(\cdot)$ to denote a mapping $S \to \pow{A}$ from states $i$ to a finite set $A(i) \subseteq A$ of available actions.
    \item $P \colon S \times A \times S \partto [0,1]$
        is a (time-independent) probabilistic transition (partial) function
        defined for all $i,j \in S$ and all $a \in A(i)$,
        and
    \item $r \colon S \times A \partto \mathbb{Q}$ is a     
        (time-independent) immediate reward (partial) 
        function defined for all $i \in S$ and all $a \in A(i)$.
\end{itemize}
We will still make use of the abbreviated notation $p_{ij}(a)$
for $P(i,a,j)$ and $r_i(a)$ for $r(i,a)$.

\subsection{Policies}
Recall that a \emph{stationary policy} $\pi^\infty$ is a sequence
$(\pi,\pi,\dots)$ where $\pi$ is a non-negative function on $S 
\times A : (i,a) \mapsto \pi_{ia}$ such that $\sum_{a \in A(i)} \pi_{ia} = 1$. Indeed, $\pi$ assigns to each $i$ a distribution over available actions.

\paragraph{Unichain (stationary) policies:} Consider the
directed graph $G(\pi^\infty) = (S,E)$ with $E \subseteq S \times S$ such that $(i,j) \in E$ if and only if $p_{ij}(a) \pi_{ia} > 0$ for some $a \in A(i)$.
We call this graph the \emph{support graph} of $\pi^\infty$. 
The edges of $G(\pi^\infty)$ represent
transitions of the MDP that have nonzero probability when
following $\pi^\infty$. We then say $\pi^\infty$ is a \emph{unichain policy from a given state $i \in S$} if the subgraph of $G(\pi^\infty)$ induced by the subset of vertices reachable from $i$ has a single maximal nontrivial strongly
connected component. (We omit references to $i$ when it is clear from the context.) For intuition, this is just a graph-based way of saying that the reachable part of the Markov chain induced by the policy has a single closed communicating class.

\begin{exercise}
    Let $\mathcal{M}$ be a communicating MDP and $i$ one of its states. Prove that there is an optimal unichain policy from $i$ for the infinite-horizon
    expected limit-average-reward criterion.
\end{exercise}

\section{Robustness of the infinite-horizon value functions} \label{sec:robustness}
In this section we will be interested in how the values of policies in
two MDPs relate to each other. In particular, we will consider
two MDPs on the same state space $S$ and action space $A$. Formally,
let us fix $\mathcal{M}_1 = (S,A,P^{(1)},r^{(1)})$ and
$\mathcal{M}_2 = (S,A,P^{(2)},r^{(2)})$. One would expect that if
$P^{(1)}$ and $P^{(2)}$, and $r^{(1)}$ and $r^{(2)}$, are not ``too
different'' then the values of a policy followed in both MDPs should not
differ by too much.

\subsection{Distance measures}
Let $f,g : A \partto [0,1]$ be two real-valued partial functions. Below, to simplify notation, we suppose $f(a) > 0$ does not hold for $a \in A$ if $f(a)$ is undefined.
We write $\dtv{f}{g}$ to denote their
\emph{total-variation distance}, i.e.
\[
    \sup_{a \in A}\{\abs{f(a) - g(a)} : f(a),g(a) > 0\}.
\]
Similarly, we write $\drat{f}{g}$ to denote
the \emph{ratio distance}: 
\[
    \sup_{a \in A}\left\{\abs{\frac{f(a)}{g(a)}},\abs{\frac{g(a)}{f(a)}} : f(a),g(a) > 0 \right\} - 1.
\]

\begin{exercise}
    Prove that $\dtv{\cdot}{\cdot}$ is a metric and that $\drat{\cdot}{\cdot}$ is not.
\end{exercise}

\subsection{Expected total discounted rewards}
Let $\alpha \in (0,1)$ be a \emph{discount factor}.
Recall that the expected total $\alpha$-discounted reward, given initial state $i$ and policy $R$ is:
\[ 
    v_i^\alpha(R) \coloneqq \Ex_{i,R}\left[\sum_{t=1}^\infty \alpha^{t-1} r_{X_t}(Y_t)\right], 
\]
where $X_t$ denotes the random variable representing the state at time $t$ and $Y_t$ the chosen action at time $t$.

For convenience, in this section we will fix some initial state $i \in S$
and write $v^\alpha$ instead of $v_i^\alpha$. To avoid ambiguities, we will
write
$v_1^\alpha$ to denote the value function for $\mathcal{M}_1$ and
$v_2^\alpha$ to denote the value function for $\mathcal{M}_2$. Furthermore,
for a stationary policy $\pi^\infty$
we will write $G_1(\pi^\infty)$ and $G_2(\pi^\infty)$ to denote the 
support graph of $\pi^\infty$ in $\mathcal{M}_1$ and $\mathcal{M}_2$
respectively.

In the sequel we will prove that, under certain conditions,
the value of any stationary policy from $\mathcal{M}_1$ when
followed in $\mathcal{M}_2$
is ``off'' by a small $\varepsilon$ only. Below, we write $\norm{r^{(1)}}{\infty}$ to denote the value $\max_{j \in S} \max_{a \in A(j)} \abs{r^{(1)}_j(a)}$.

\begin{theorem}\label{thm:disc-robust}
    Let $\alpha,\varepsilon \in (0,1)$. Then for all stationary policies
    $\pi^\infty$ we have that
    \(
        (1-\alpha)\abs{v^\alpha_1(\pi^\infty) - v^\alpha_2(\pi^\infty)}
        \leq \frac{\varepsilon}{2} + \frac{\varepsilon}{2}\norm{r^{(1)}}{\infty}
    \)
    if the following hold.
    \begin{enumerate}[label={(A\arabic*)}]
        \item $G_1(\pi^\infty) = G_2(\pi^\infty)$ \label{itm:supp}
        \item $\drat{P^{(1)}}{ P^{(2)}} \leq \frac{\varepsilon}{4\exp(2|S|)}$ \label{itm:drat-prob}
        \item $\dtv{r^{(1)}}{r^{(2)}} \leq \frac{\varepsilon}{2}$ \label{itm:dtv-reward}
    \end{enumerate}
\end{theorem}
A proof of the theorem is given in the following subsection. The argument presented there requires some further definitions and some intermediate results. Later, we will state and prove a similar result for the limit-average reward. For now, it is interesting to note that the following is an immediate corollary of the theorem that applies when rewards are non-negative values from $[0,1]$.
\begin{corollary}
    Let $\alpha,\varepsilon \in (0,1)$ and suppose 
    $r^{(1)}_j(a) \in [0,1]$ for all $j \in S$ and all $a \in A(j)$. Then for all stationary policies
    $\pi^\infty$ we have that
    \(
        (1-\alpha)\abs{v^\alpha_1(\pi^\infty) - v^\alpha_2(\pi^\infty)}
        \leq \varepsilon
    \)
    if \autoref{itm:supp}, \autoref{itm:drat-prob}, and \autoref{itm:dtv-reward} hold.
\end{corollary}

\begin{exercise}
    Argue that \autoref{itm:supp} in \autoref{thm:disc-robust} is necessary. That is,
    construct MDPs that satisfy the other conditions but not \autoref{itm:supp} and such that policy values are not robust.
\end{exercise}

\subsection{A proof of \autoref{thm:disc-robust}}
We present an argument due to Eilon Solan~\cite{solan03} and later refined by Krishnendu Chatterjee~\cite{chatterjee12}. The crux of the proof lies in showing that $v_1^\alpha$ and $v_2^\alpha$ have a ``closed form'' as a quotient of polynomials. That is, if we see the transition probabilities $p_{ij}(a)$ as variables $x_{ij}^a$, then the expected total $\alpha$-discounted reward is given by the quotient of two polynomial functions.
\begin{lemma}\label{lem:rational-quot}
    For all stationary policies $\pi^\infty$ and all $\alpha \in (0,1)$ there exist polynomials $f(\vec{x}),g(\vec{x})$
    of degree at most $|S|$, with non-negative coefficients, and such that:
    \(
        v_1^\alpha(\pi^\infty) = \nicefrac{f(P^{(1)})}{g(P^{(1)})} \text{ and }
        v_2^\alpha(\pi^\infty) = \nicefrac{f(P^{(2)})}{g(P^{(2)})}.
    \)
\end{lemma}

Using the bounds on the degree of the polynomials, together with the non-negativity of their coefficients, one can prove the following bound from which \autoref{thm:disc-robust} will (almost immediately) follow. For convenience, we write below $a \leq \nicefrac{b}{c} \leq d$ to denote $ac \leq b$ and $b \leq dc$. (Intuitively $\nicefrac{0}{0}$ can be seen as $1$ in such inequalities.)
\begin{lemma}\label{lem:poly-bnd}
    Let $f(x_1,\dots,x_n)$ be a polynomial function of degree $d$ with non-negative coefficients, $\varepsilon \in \mathbb{R}$ with $\varepsilon > 0$, and $\vec{a},\vec{b} \in \mathbb{R}^n$ with $\vec{a},\vec{b} \geq \vec{0}$. If $(1+\varepsilon)^{-1} \leq \nicefrac{\vec{a}_i}{\vec{b}_i} \leq 1 + \varepsilon$ for all $1 \leq i \leq n$ then $(1+\varepsilon)^{-d} \leq \nicefrac{f(\vec{a})}{f(\vec{b})} \leq (1 + \varepsilon)^d$.
\end{lemma}

\begin{exercise}
    Prove the bound.
\end{exercise}

In the following subsections we develop auxiliary lemmas that allow us to prove~\autoref{thm:disc-robust}. Then, in~\autoref{sec:proof}, we provide the global argument which makes use of said lemmas. The starting point for the argument is the observation that the \emph{reachability probability} is
a rational function --- i.e. a quotient of polynomial functions. This is made precise in the sequel. 

\subsubsection{A rational hitting function} Let $(S,\delta)$ be a Markov chain with $S$ the states and $\delta \colon S \to [0,1]$ the probabilistic transition function. We will be interested in the probability of the chain hitting states from $Q \subset S$. Denote by $\tau(Q) \coloneqq \inf\{t \in \mathbb{N} \mid X_t \in Q\}$ the first $Q$-hit time. Note that $\tau(Q)$ is a random variable whose value ranges from $0$ to $+\infty$. Write $\mathcal{F}$ to denote the set of functions $\{f : S \setminus Q \to S\}$ and let $G_f = (S,E_f)$ be the directed graph where $(i,j) \in E$ if and only if $f(i) = j$. For every $f \in \mathcal{F}$ we define $\alpha_f \in \{0,1\}$ and $\beta_f \colon S \times S \to \{0,1\}$ as follows:
\begin{itemize}
    \item We set $\alpha_f$ to $1$ if $G_f$ is a directed acyclic graph and to $0$ otherwise;
    \item We set $\beta_f(i,j)$ to $1$ if the unique path starting from $i$ in $G_f$ visits $j$ and to $0$ otherwise.
\end{itemize}

The following is usually attributed to Freidlin and Wentzell~\cite{fw98}. This specific version of the claim is explicitly proved in~\cite{cc97}.
\begin{lemma}\label{lem:reach-poly}
    If $\Pr_i(\tau(Q) < +\infty) > 0$ for all $i \in S$ then for all $j \not\in Q$ and all $k \in Q$ we have that $\Pr_j(X_{\tau(Q)} = k)$ is equal to the following.
    \[
        \frac{\sum_{f \in \mathcal{F}} \beta_f(j,k) \prod_{\ell \not\in Q} \delta(\ell,f(\ell))}{\sum_{f \in \mathcal{F}} \alpha_f \prod_{\ell \not\in Q} \delta(\ell,f(\ell))}
    \]
\end{lemma}

The fact that the reachability probability is a rational function is not too surprising since it is the unique solution of a linear system of equations $\vec{A} \cdot \vec{x} = \vec{b}$. Indeed, even if coefficients in $\vec{A}$ and $\vec{b}$ have polynomial functions as entries instead of rational numbers, the system can be ``symbolically'' solved using Cramer's rule and Leibniz' formula to obtain a rational function as the ``solution'' of the system (cf.~\cite{junges20,bhhjkk20,jkow21}). However, the above statement yields one with the property that all coefficients in the numerator and the denominator are positive. The statement can also be seen as a way to compute the determinant of a (stochastic) matrix by ``counting subgraphs'' induced by it (cf.~\cite{greenman76}).

\subsubsection{A rational discounted-reward function}
Using \autoref{lem:reach-poly}, we will now obtain a rational-function representation for the expected total discounted reward function for Markov chains. Again, we will view the transition probabilities as unknown parameter values. Concretely, let $S$ be a set of states and $\vec{x} = (x_{11}, x_{12}, \dots, x_{ij}, \dots)$ be a tuple of $|S|^2$ variables. Then, for a polynomial $f(\vec{x})$ over $\vec{x}$ and a probabilistic transition function $\delta$ such that $(S,\delta)$ is a Markov chain, we write $f(\delta)$ to denote $f(\delta_{11}, \delta_{12}, \dots, \delta_{ij}, \dots)$. Finally, below, we will use the following definition. For states $i,j$ from a Markov chain,
\[ 
    \mathrm{MT}_{i,j}^\alpha \coloneqq \Ex_{i}\left[(1 - \alpha)\sum_{t=1}^\infty \alpha^{t-1} \mathbf{1}_{X_t = j}\right],
\]
where $\mathbf{1}_{X_t = j}$ is the indicator function,
denotes the \emph{expected discounted time} spent at state $j$.

\begin{lemma}\label{lem:disc-time}
    For all $i,j \in S$ and all $\alpha \in (0,1)$ there exist two polynomials $f(\vec{x})$ and $g(\vec{x})$ such that the following hold.
    \begin{enumerate}
        \item Both $f(\vec{x}),g(\vec{x})$ have degree at most $|S|$ and non-negative coefficients.
        \item Let $(S,\delta)$ be a Markov chain. Then, $\mathrm{MT}_{i,j}^\alpha = \nicefrac{f(\delta)}{g(\delta)}$.
    \end{enumerate}
\end{lemma}
\begin{proof}
    Without going into all the mathematical details, we presently sketch the argument given in~\cite{chatterjee12} to establish the result. 
    Let $i,j$, $\alpha$, and $\delta$ be arbitrary. We will construct a Markov chain $\mathcal{C}$ obtained by duplicating the set $S$ of states. That is, its states $S'$ include $S$ and a set of copies $S_1 = \{k_1 \mid k \in S\}$. The Markov chain will be such that $\Pr_i(X_{\tau(S')}=j_1)$ coincides with the value $\mathrm{MT}_{i,j}^\alpha$ from the claim. The desired result will thus follow from \autoref{lem:reach-poly} since the $\alpha_f$ and $\beta_f(\cdot,\cdot)$ terms do not depend on $\delta$ and the monomials in the numerator and denominator of the given formula are each products of at most $|S|$ terms of the form $\delta(\ell,f(\ell))$.
    
    The probabilistic transition relation $\delta'$ of $\mathcal{C}$ is defined as follows. First, all copy states are absorbing: $\delta'(k_1,k_1) = 1$ for all $k_1 \in S_1$. Second, for all $k \in S$, the next state is the copy $k_1$ with probability $1-\alpha$ or a state $\ell$ from $S$ with probability $\alpha \cdot \delta(k,\ell)$. This is repeated in symbols below.
    \[
        \delta'(k,\ell) = \begin{cases}
        1-\alpha & \text{if } \ell = k_1\\
        \alpha \cdot \delta(k,\ell) & \text{if } \ell \in S\\
        0 & \text{if } \ell \in S_1 \setminus \{k_1\}
        \end{cases}
    \]
    Note that, by construction, we have that $\Pr_k(\tau(S') < +\infty) > 0$ for all $k \in S$. Indeed, we have that $\delta'(k)(k_1) = 1 - \alpha > 0$ so: 
    \[
    \Pr_k(\tau(S') < +\infty) \geq \Pr_k(\tau(S') = 1) = 1 - \alpha > 0.
    \]
    To conclude the proof, it remains to establish the relation between $\mathrm{MT}_{i,j}^\alpha$ and $\Pr_i(X_{\tau(S')}=j_1)$ in $\mathcal{C}$ that was mentioned before. In this direction, we first observe that $(\Pr_k(X_\tau{(S')}=j_1))_{k \in S}$ can be obtained as the unique solution of the following system of (Bellman optimality) linear equations:
    \begin{align*}
        & y_k = (1-\alpha)\mathbf{1}_{k=j} + \sum_{\ell \in S} \alpha \cdot \delta(k,\ell) \cdot y_\ell
        & \text{for all } k \in S
    \end{align*}
    Furthermore, it is not hard to prove that the same holds for $(\mathrm{MT}_{k,j}^\alpha)_{k \in S}$, which thus concludes the proof.
\end{proof}

The above result already gives us a proof for \autoref{lem:rational-quot}. 
\begin{proof}[Proof of \autoref{lem:rational-quot}]
    Observe that $(1-\alpha) v_1^\alpha(\pi^\infty) = \sum_{j \in S} \sum_{a \in A(j)} r_j(\pi_{ja}) \cdot \mathrm{MT}^\alpha_{i,j}$
    where the expectation in the definition of $\mathrm{MT}^\alpha_{i,j}$ is taken on the Markov chain induced by following $\pi^\infty$ in the MDP $\mathcal{M}_1$. The analogue holds for $(1-\alpha) v_2^\alpha(\pi^\infty)$. \autoref{lem:rational-quot} thus follows from \autoref{lem:disc-time}.
\end{proof}

\subsubsection{Putting everything together}\label{sec:proof}
Now that we have proved \autoref{lem:rational-quot}, we will make use of the bound from \autoref{lem:poly-bnd} to conclude the proof of \autoref{thm:disc-robust}.
\begin{proof}[Proof of \autoref{thm:disc-robust}]
    We begin with some simplifying preliminaries. First, we will define the expected total discounted reward obtained by $\pi^\infty$ in $\mathcal{M}_2$ if we replace $r^{(2)}$ by $r^{(1)}$. This will allow us to deal with a single reward function.
    In symbols:
    \[
        v_3^\alpha(\pi^\infty) \coloneqq \Ex_{i,\pi^\infty}
        \left[ \sum_{t=1}^\infty \alpha^{t-1} r^{(1)}_{X_t}(Y_t) \right]
    \]
    where the expectation is taken with respect to the probability space induced by $\mathcal{M}_2$ and $\pi^{\infty}$. The inequality below follows from the definition.
    \[
        \abs{v_2^\alpha(\pi^\infty) - v_3^\alpha(\pi^\infty)} \leq \frac{\dtv{r^{(1)}}{r^{(2)}}}{1-\alpha}
    \]
    Hence, it will suffice to prove that $2(1-\alpha)\abs{v_1^\alpha(\pi^\infty) - v_3^\alpha(\pi^\infty)} \leq \varepsilon \norm{r^{(1)}}{\infty}$. Indeed, by
    \autoref{itm:dtv-reward} the desired result follows from the latter inequalities. We introduce one further simplification to make sure the expected total discounted rewards can be assumed to be positive. Let $\rho \coloneqq -\norm{r^{(1)}}{\infty}$ be a lower bound for rewards from $r^{(1)}$. We will ``shift'' the reward function to make sure all rewards are positive. Formally, define $r^{(3)}$ such that $r^{(3)}_j(a) = r^{(1)}_j(a) - \rho$ for all $j \in S$ and all $a \in A(j)$. We will henceforth focus on the value functions $w_1^\alpha(\pi^\infty)$ and $w_2^\alpha(\pi^\infty)$ defined as:
    \(
      \Ex_{i,\pi^\infty}\left[ \sum_{t=1}^\infty \alpha^{t-1} r^{(3)}_{X_t}(Y_t) \right]
    \)
    with respect to the probability spaces induced by $\mathcal{M}_1$ and $\pi^\infty$ and $\mathcal{M}_2$ and $\pi^\infty$, respectively. It is easy to see that by definition we have:
    \[
      w_1^\alpha(\pi^\infty) = v_1^\alpha(\pi^\infty) - \frac{\rho}{1-\alpha} \text{ and }
      w_2^\alpha(\pi^\infty) = v_3^\alpha(\pi^\infty) - \frac{\rho}{1-\alpha}
    \]
    and that therefore the following equality holds.
    \[
      \abs{w_1^\alpha(\pi^\infty) - w_2^\alpha(\pi^\infty)} = \abs{v_1^\alpha(\pi^\infty) - v_3^\alpha(\pi^\infty)}
    \]
    By the arguments above, to prove the claim it suffices to show that \begin{equation}\label{eqn:solan-ineq}
    \abs{w_1^\alpha(\pi^\infty) - w_2^\alpha(\pi^\infty)} \leq \frac{\varepsilon \norm{r^{(1)}}{\infty}}{2(1-\alpha)}.
    \end{equation}
    This is precisely what is done below.
    
    By \autoref{lem:rational-quot} we have that there exist polynomials $f(\vec{x}),g(\vec{x})$ of degree at most $|S|$ and non-negative coefficients such that the following holds.
    \[
    \frac{w_1^\alpha(\pi^\infty)}{w_2^\alpha(\pi^\infty)}
    =
    \frac{f(P^{(1)}) g(P^{(2)})}{g(P^{(1)})f(P^{(2)})}
    \]
    That is, there is a polynomial $h \colon (\vec{x},\vec{y}) \mapsto f(\vec{x})g(\vec{y})$ of degree at most $2|S|$ and non-negative coefficients such that: 
    \[
    \frac{w_1^\alpha(\pi^\infty)}{w_2^\alpha(\pi^\infty)} =  \frac{h(P^{(1)},P^{(2)})}{h(P^{(2)},P^{(1)})}.
    \]
    Let $\delta = \drat{P^{(1)}}{P^{(2)}}$. Then, the following inequalities then follow from \autoref{lem:poly-bnd}.
    \begin{equation}\label{eqn:w-bnds}
        (1 + \delta)^{-2|S|} \leq \frac{w_1^\alpha(\pi^\infty)}{w_2^\alpha(\pi^\infty)} \leq (1 + \delta)^{2|S|}
    \end{equation}
    
    Using \autoref{eqn:w-bnds}, we will now argue that the following inequality holds. 
    \begin{equation}\label{eqn:chatterjee-bnd}
      \abs{w_1^\alpha(\pi^\infty) - w_2^\alpha(\pi^\infty)} \leq \left(\frac{2\norm{r^{(1)}}{\infty}}{1 - \alpha}\right)\left(\left(1 + \drat{P^{(1)}}{P^{(2)}}\right)^{2|S|} - 1\right)
    \end{equation}
    Let us focus on the case when $w_1^\alpha(\pi^\infty) \geq w_2^\alpha(\pi^\infty)$, the other case is symmetrical. Note that we can assume $w_2^\alpha(\pi^\infty) > 0$ lest $w_1^\alpha(\pi^\infty) = w_2^\alpha(\pi^\infty) = 0$ and the inequality holds. Indeed, if $w_2^\alpha(\pi^\infty) = 0$, since the rewards from $r^{(3)}$ are all non-negative, it must be that all edges reachable from the initial state $i$ in $G_1(\pi^\infty)$ --- and thus $G_2(\pi^\infty)$ --- correspond to zero-reward transitions. We thus have the following.
    \begin{align*}
        & \abs{w_1^\alpha(\pi^\infty) - w_2^\alpha(\pi^\infty)}\\
        {} = {} & w_2^\alpha(\pi^\infty) \left(\frac{w_1^\alpha(\pi^\infty)}{w_2^\alpha(\pi^\infty)} -1 \right) & \text{well def'd since } w_2^\alpha(\pi^\infty) > 0\\
        {} \leq {} & w_2^\alpha(\pi^\infty) \left( (1 + \delta)^{2|S|}) -1 \right) & \text{by \autoref{eqn:w-bnds}} 
    \end{align*}
    Since $\norm{r^{(3)}}{\infty} \leq 2\norm{r^{(1)}}{\infty}$, we have that $w_1^\alpha(\pi^\infty), w_2^\alpha(\pi^\infty) \leq 2\norm{r^{(1)}}{\infty}(1-\alpha)^{-1}$ and \autoref{eqn:chatterjee-bnd} follows from our choice of $\delta$. 
    
    To conclude, note that the following inequalities hold.
    \begin{align*}
        \left( 1 + \delta \right) ^{2|S|} - 1 = {} & \sum_{k=1}^{n=2|S|} \binom{n}{k} \delta^k & \text{by the binomial expansion}\\
        {} \leq {} & \sum_{k=1}^{n=2|S|} \binom{n}{k} \delta & \text{since }
        \delta < 1 \text{ by \autoref{itm:drat-prob}}\\
        {} \leq {} & 2^{2|S|} \delta  \leq \exp(2|S|) \delta & \\
        {} \leq {} & \frac{\varepsilon \exp(2|S|)}{4\exp(2|S|)} = \frac{\varepsilon}{4} & \text{choice of } \delta \text{ and \autoref{itm:drat-prob}} 
    \end{align*}
    Hence, by \autoref{eqn:chatterjee-bnd} we have:
    \[
        \abs{w_1^\alpha(\pi^\infty) - w_2^\alpha(\pi^\infty)}
        \leq \left(\frac{\varepsilon}{2}\right) \left(\frac{\norm{r^{(1)}}{\infty}}{1 - \alpha}\right)
    \]
    as required.
\end{proof}

\subsection{Expected limit average rewards}\label{sec:robustness-avg}
Recall that the expected (lower) limit average reward, given initial state $i$ and policy $R$ is:
\[ 
    \phi_i(R) \coloneqq \Ex_{i,R}\left[\liminf_{T\to \infty} \frac{1}{T} \sum_{t=1}^T r_{X_t}(Y_t)\right], 
\]
where $X_t$ denotes the random variable representing the state at time $t$ and $Y_t$ the chosen action at time $t$. We also write $\phi = \sup_R \phi(R)$ and omit the $i$ in $\phi_i$ if it is clear from the context.

Mertens and Neyman originally proved the following equivalence~\cite{mn81}.
\begin{lemma}\label{lem:rel-ds-mp}
    Consider an MDP. For all stationary policies $R$ we have that
    \(\phi(R) = \lim_{\alpha \to 1} (1-\alpha) v^\alpha(R).\)
\end{lemma}

\subsubsection{An analogue of~\autoref{thm:disc-robust}}
As before, we will fix some initial state $i \in S$
and write $\phi$ instead of $\phi_i$. To avoid ambiguities, we will
write
$\phi_1$ to denote the value function for $\mathcal{M}_1$; and
$\phi_2$, the value function for $\mathcal{M}_2$.

The result for expected limit average rewards 
follows from \autoref{lem:rel-ds-mp} and \autoref{thm:disc-robust}.
\begin{theorem}\label{thm:mp-robust}
    Let $\varepsilon \in (0,1)$. Then for all stationary policies
    $\pi^\infty$ we have that
    \(
        \abs{\phi_1(\pi^\infty) - \phi_2(\pi^\infty)}
        \leq \frac{\varepsilon}{2} + \frac{\varepsilon}{2}\norm{r^{(1)}}{\infty}
    \)
    if the following hold.
    \begin{enumerate}[label={(A\arabic*)}]
        \item $G_1(\pi^\infty) = G_2(\pi^\infty)$
        \item $\drat{P^{(1)}}{ P^{(2)}} \leq \frac{\varepsilon}{4\exp(2|S|)}$
        \item $\dtv{r^{(1)}}{r^{(2)}} \leq \frac{\varepsilon}{2}$
    \end{enumerate}
\end{theorem}
\begin{proof}
  The result is proved using the following equalities.
  \begin{align*}
    & \abs{\phi_1(\pi^\infty) - \phi_2(\pi^\infty)} & \\
    {} = {} & \abs{\lim_{\alpha \to 1} (1-\alpha)v^\alpha_1(\pi^\infty) - \lim_{\alpha
    \to 1} (1-\alpha) v^\alpha_2(\pi^\infty)} & \text{by
    \autoref{lem:rel-ds-mp}} \\
    {} = {} & \abs{\lim_{\alpha \to 1} (1-\alpha)(v^\alpha_1(\pi^\infty) -
    v^\alpha_2(\pi^\infty))} & \\
    {} = {} & \lim_{\alpha \to 1} \abs{(1-\alpha)(v^\alpha_1(\pi^\infty) -
    v^\alpha_2(\pi^\infty))} & \\
    {} = {} & \lim_{\alpha \to 1} (1-\alpha) \abs{v^\alpha_1(\pi^\infty) -
    v^\alpha_2(\pi^\infty)} & \text{since } \alpha \leq 1
  \end{align*}
  In the last equality above, by \autoref{thm:disc-robust} every
  element in the (limit of the) sequence is at most the following value.
  \[
    \left( 1 - \alpha \right)
    \left(
    \frac{\varepsilon}{2(1-\alpha)} +
    \frac{\varepsilon}{2(1-\alpha)}\norm{r^{(1)}}{\infty}
    \right) =
    \frac{\varepsilon}{2} +
    \frac{\varepsilon}{2}\norm{r^{(1)}}{\infty}
  \]
  Hence, the value of the limit is at most
  \( \frac{\varepsilon}{2} + \frac{\varepsilon}{2}\norm{r^{(1)}}{\infty} \)
  as required.
\end{proof}

The above result tells us that we can get ``good'' policies for MDPs with unknown dynamics. Indeed, we just need to explore the MDP to approximate the probabilities and rewards, and then apply any policy-synthesis algorithm that we want. We formalize this approach in the rest of this notes.

\section{Model-based limit-average optimization}
In this section we will present a learning algorithm for limit-average optimization based on two papers~\cite{kpr18,kmmp20}. It consists in bringing the previous results to practice: We will execute ever longer episodes consisting of \emph{exploration phases} followed by \emph{exploitation phases}. The exploration phases will allow us to compute good approximations of the unknown MDP while the exploitation phase will follow an optimal policy for the approximated MDP. Importantly, the exploitation phase has to compensate for the potentially sub-optimal exploration phase.
Before we present the algorithm we have to recall some preliminaries regarding tail bounds.

\subsection{Notation and assumptions}
As in \autoref{sec:robustness}, we will consider two MDPs $\mathcal{M}_1$ and $\mathcal{M}_2$. Intuitively, $\mathcal{M}_1$ is the real environment with which our algorithm is interacting and will compute approximations $\mathcal{M}_2$ of it as a side product of the algorithm. To avoid ambiguities, we recall the formal definitions: $\mathcal{M}_1 = (S,A,P^{(1)},r^{(1)})$ and
$\mathcal{M}_2 = (S,A,P^{(2)},r^{(2)})$.
Further, as in \autoref{sec:robustness-avg}, we fix some initial state $i \in S$
and write $\phi$ instead of $\phi_i$. To avoid ambiguities, we will
write
$\phi_1$ to denote the value function for $\mathcal{M}_1$; and
$\phi_2$, the value function for $\mathcal{M}_2$.

\paragraph{Assumptions.} In this section we will make the assumption that we have one additional piece of information regarding $\mathcal{M}_1$. Namely, we are given $\pmin \in \mathbb{Q}$ such that $0 < \pmin \leq 1$ and $0< p^{(1)}_{ij}(a) \implies \pmin \leq p^{(1)}_{ij}(a)$ for all $i,j \in S$ and all $a \in A(i)$. In words, $\pmin$ is a lower bound for the (unknown) non-zero transition probabilities from $\mathcal{M}_1$. We will also assume that the rewards are bounded by $W$: i.e. $|r_i(a)| \leq W$ for all $i \in S$ and all $a \in A(i)$.
Additionally, we will assume that $\mathcal{M}_1$ is communicating --- since otherwise there is no way we can learn. Recall that $\mathcal{M}_1$ is communicating (or an \emph{end component}) if for every $i,j \in S$ there exists a deterministic stationary policy $f \in C(D)$ such that $j$ is accessible from $i$ under the policy $f$. In other words, there exists $n \geq 0$ such that $(P^{(1)}(f))^n_{ij} > 0$.

\subsection{Tail bounds for MDP exploration}
To compute $\mathcal{M}_2$, we will follow a stationary policy $(\rho,\rho, \dots)$ which --- roughly speaking --- executes ``uniform random exploration''. Formally, the policy $\rho$ is $S \times A : (i, a) \mapsto \nicefrac{1}{|A(i)|}$ for all $i \in S$. Intuitively, each time the MDP enters a state $i$ then $\rho$ performs an experiment by playing $\rho(i,a)$ uniformly at random (over $a \in A(i)$) and observes both the resulting successor state as well as the immediate reward. Note that, since we have assumed $\mathcal{M}_1$ is communicating, such a policy $\rho$ always exists. Further, if we follow it for an infinite number of steps then we must see all state-action pairs infinitely often with probability $1$. Indeed, this follows, for instance, from the second Borel-Cantelli lemma and our definition of communicating MDPs. However, since we plan on exploiting for some time after having explored, we cannot continue following $\rho$ forever. A natural question arises: how long should we follow $\rho$ before we are sure that we can construct $P^{(2)}$ and $r^{(2)}$ so that $\mathcal{M}_1$ and $\mathcal{M}_2$ are not ``too different''.

\begin{lemma}\label{lem:hoeffding-mdps}
    For all $\varepsilon,\delta \in (0,1)$ one can compute $n \in \mathbb{N}$ such that following $\rho$ for $n$ steps suffices to collect enough statistics to compute $\mathcal{M}_2$ so that \autoref{itm:supp}, \autoref{itm:drat-prob}, and \autoref{itm:dtv-reward} from \autoref{thm:mp-robust} all hold with probability at least $1 - \delta$.
\end{lemma}

Before we give some intuition on how to prove the above lemma, we establish a relation between the distance measures mentioned in it. The following is a simple observation due to Chatterjee~\cite{chatterjee12}.
\begin{lemma}\label{lem:rel-dists}
    Let $\mathcal{M}_2$ be such that \autoref{itm:supp} from \autoref{thm:mp-robust} holds. Then, it holds that
    \(
        \drat{P^{(1)}}{P^{(2)}} \pmin \leq \dtv{P^{(1)}}{P^{(2)}}.
    \)
\end{lemma}
\begin{exercise}
    Prove \autoref{lem:rel-dists}.
\end{exercise}

It follows from the above lemma that to prove
\autoref{lem:hoeffding-mdps} we can focus on reducing the total-variation distance of the $P^{(i)}$ and $r^{(i)}$. Thankfully, this is precisely what tail bounds are great for. Indeed, since $\mathcal{M}_1$ is communicating and we have a lower bound $\pmin$ on the non-zero transition probabilities, we know that the following holds.
\[
    (P^{(1)}(\rho))^{|S|}_{ij} \geq \pmin^{|S|}
\]
The rest of the proof is a straightforward application of Hoeffding's inequality which is left to the reader as an exercise.

\begin{exercise}
    Prove~\autoref{lem:hoeffding-mdps}.
\end{exercise}

\subsection{Exploring and exploiting ad infinitum}
We now present a policy $R$ that ensures an optimal limit average with probability $1$. To avoid clutter, we describe the policy in words. However, from the description it is straightforward to formalize its definition. The policy $R$ plays in \emph{episodes} sub-divided into repeated \emph{exploration} and \emph{exploitation phases} as follows.
\begin{description}
    \item[Explore:] First, $R$ follows a policy $\rho$ to realize uniform random exploration. During episode $i$, $\rho$ is followed for $L_i$ steps and $\mathcal{M}_2^{(i)}$ is constructed from the collected statistics.
    \item[Exploit:] Second, $R$ follows an optimal unichain policy $\pi^\infty$ w.r.t. to the expected limit average in $\mathcal{M}_2^{(i)}$. The latter is done for
    $O_i$ steps.
\end{description}

In this section we will argue that $R$ is almost-surely optimal.
\begin{theorem}\label{thm:as-optimal}
    Let $j \in S$. For all $i \in \mathbb{N}$ one can compute $L_i,O_i \in \mathbb{N}$ so that the following holds.
    \[
        \Pr_{j,R}\left(\liminf_{T \to \infty} \frac{1}{T} \sum_{t=1}^T r_{X_t}(Y_t) \geq \phi_j \right) = 1
    \]
\end{theorem}

\autoref{lem:hoeffding-mdps} gives us a way of computing $L_i$ so that $\mathcal{M}_2^{(i)}$ is arbitrarily close to $\mathcal{M}_1$ with as probability as close to $1$ as desired. Then, because of \autoref{thm:mp-robust}, we know that optimizing the expected limit-average reward in $\mathcal{M}_2^{(i)}$ means we are near-optimal in $\mathcal{M}_1$. (Indeed, we do not claim this is the case with probability $1$. Rather, we can do so with probability $1-\delta$ only because of the approximation guarantee from \autoref{lem:hoeffding-mdps}.) If the sequences of $L_i,O_i$ are increasing then the result should intuitively follow. The one remaining issue is that we need each $O_i$ to be finite! Furthermore, it needs to account for all preceding episodes and get us closer to the optimal expected limit average in $\mathcal{M}_2^{(i)}$.

\subsubsection{Convergence speed to optimality}
Our goal is to prove the following result regarding unichain Markov chains $\mathcal{C} = (S,P,r)$. It will allow us to compute $O_i$ as desired. 
\begin{lemma}[Tracol's lemma]\label{lem:tracol}
    Let $\mathcal{C}$ be a finite unichain Markov chain and $\pmin$ be a lower bound for all nonzero transition probabilities in $\mathcal{C}$. For all $\varepsilon \in (0,1)$ one can compute $K_0 \in \mathbb{N}$ and $\alpha,\beta \in \mathbb{Q}$ (from $|S|, \pmin,$ and $\varepsilon$ alone) such that $\alpha,\beta > 0$ and the following holds for all $T \geq K_0$ and all $i \in S$.
    \[
    \Pr_i\left( \sum_{t=1}^T r_{X_t}(Y_t) \geq \Ex_i\left[ \sum_{t=1}^T r_{X_t}(Y_t) \right] - T\varepsilon \right) \geq 1 - \alpha e^{-T \beta \varepsilon^2}
    \]
\end{lemma}

The convergence lemma was first proved by Mathieu Tracol~\cite{tracol09}. We state it here in a slightly more general way since we need to be able to compute the constants from the limited information we have available to us. (Namely, we cannot use the transition probabilities!) Before we sketch an argument to prove it, we have to make a digression into mixing times and uniform ergodicity.

\subsubsection{Mixing times and uniform ergodicity}
 We say the Markov chain $\mathcal{C}$ is \emph{regular} if $|Q|$ is finite and it is both aperiodic and irreducible. Recall that a state $i \in S$ has period $\gcd\{t > 0 : \Pr(X_t=i) > 0\}$. The Markov chain is said to be aperiodic if the period of every state is $1$; it is irreducible if $(S,\{(i,j) \in S^2 \mid P(i,j) > 0\})$ is a strongly connected directed graph. Further recall that we write $P^*$ for the stationary matrix of a regular Markov chain.
 
We recall a sufficient condition for a Markov chain to be \emph{uniformly ergodic}~\cite{mt93}: $\mathcal{C}$ satisfies \emph{Doeblin's condition} if there exist $\lambda \in \mathbb{R}, \lambda > 0$, a probability measure $\varphi$ over (subsets of) $S$, and an integer $t \in \mathbb{T}$, such that:
\[
    \Pr_i(X_t \in T) \geq \lambda \varphi(T)
\]
for all $T \subseteq S$ and all $i \in S$. The constant $\lambda$ is usually called the \emph{ergodicity coefficient} and $t$ the \emph{mixing time}.

It is easy to prove that Doeblin's condition holds for all regular Markov chains. For completeness, we provide an argument to this effect below while highlighting that both the ergodicity coefficient and the mixing time can be computed even in the absence of the transition probabilities.
\begin{lemma}\label{lem:uniform-ergo}
    Let $\mathcal{C}$ be a regular Markov chain and let $\pmin$ be a lower bound for all of its nonzero transition probabilities. One can compute $\lambda \in \mathbb{Q}, \lambda > 0$ and $t \in \mathbb{N}$ (from $|S|$ and $\pmin$ alone) such that \(\Pr_i(X_t \in T) \geq \lambda P^*(T)\) for all $T \subseteq S$ and all $i \in S$.
\end{lemma}
\begin{proof}
    First, since $\mathcal{C}$ is regular, we know it has a unique stationary distribution with matrix $P^*$.  We now observe that, because of aperiodicity and irreducibility, we know there exists some $t$ such that $\Pr_i(X_t = j) > 0$ for any $i,j \in S$. To give an explicit definition of such a $t$ we need to recall some definitions.
    
    Given a finite set $N = \{a_1,\dots,a_\ell\}$ of positive integers suc that $\gcd(N) = 1$, we write $g(N)$ to denote its Frobenius number. That is, the maximal integer that cannot be obtained as a conical combination of the $a_i$, i.e. as a sum of the form:
    \(
        \sum_{x=1}^\ell k_ia_i
    \)
    where $k_1,\dots,k_\ell \in \mathbb{N}$. The existence of $g(N)$ is guaranteed by Schur's theorem which gives a bound on the amount of numbers that cannot be obtained as such conical combinations. We thus set:
    \(
        t \coloneqq \max_{N \subseteq S} g(N) + 1.
    \)
    Since $\mathcal{C}$ is regular, we have that:
    \[
        \forall i,j \in S: \Pr_i(X_t = j) \geq \pmin^t > 0.
    \]
    To conclude, we set $\lambda \coloneqq \pmin^t$. Hence, $\Pr_i(X_t \in T) \lambda^{-1} \geq 1 \geq P^*(T)$ for all $T \subseteq S$ and all $i \in S$ as required.
\end{proof}

\subsubsection{A sketch of Tracol's argument}
We briefly describe how \autoref{lem:tracol}
can be proved. The original proof by Tracol consists in decomposing $\mathcal{C}$ into its transient set of states and the regular sub-chain $\mathcal{C}'$ it contains. Then, one can further decompose $\mathcal{C}'$ into aperiodic components based on all the residue classes modulo the period of states in $\mathcal{C}'$. Finally, one can use \autoref{lem:uniform-ergo} to obtain bounds for each such regular chain. Critically, $K_0, \alpha, \beta$ can all be computed as a function of $|S|, \pmin$, the ergodicity coefficient of each regular sub-chain, as well as their mixing times. Since $\mathcal{C}$ is finite and the period of $\mathcal{C}'$ is bounded by $|S|$, we can compute $K_0,\alpha,\beta$ taking all possible such decompositions into account. (This allows us the flexibility of not having to rely on the actual transition-probability values, which we do not have!)

\begin{exercise}
    Use the optimality equations for the expected limit-average-reward criterion and the convergence bounds that follow from \autoref{lem:uniform-ergo} to prove
    \autoref{lem:tracol}.
\end{exercise}

\subsubsection{Using Tracol's convergence lemma}
We will now use Tracol's lemma to prove the following properties hold for policies used in $\mathcal{M}_1$.
\begin{lemma}\label{lem:uni-strats}
    For all $j \in S$ and all unichain policies $\pi^\infty$ the following hold.
    \begin{enumerate}
        \item \(
        \Pr_{j,\pi^\infty}\left(\liminf_{T \to \infty} \frac{1}{T} \sum_{t=1}^T r_{X_t}(Y_t) \geq \phi_j(\pi^\infty)\right) = 1 \) \label{itm:exp-asurely}
        \item For all $\varepsilon \in (0,1)$ one can compute $M$ (from $|S|$, $|A|$, and $\pmin$) we have: \[\Pr_{j,\pi^\infty}\left(\forall N \geq M: \sum_{t=1}^N r_{X_t}(Y_t) \geq \Ex_{j,\pi^\infty}\left[ \sum_{t=1}^N r_{X_t}(Y_t) \right] - N\varepsilon \right) = 1 - \varepsilon.\] \label{itm:long-convergence}
    \end{enumerate}
\end{lemma}

The first item in the above lemma is folklore. We give a self-contained proof of it based on the second item. Thus, let us first argue the latter holds.
\begin{proof}[Proof of \autoref{lem:uni-strats}  \autoref{itm:long-convergence}]
    Recall the bound from \autoref{lem:tracol} and let $K_0$ be the corresponding integer computed for $\varepsilon$ and the Markov chain induced by $\pi^\infty$ and $\mathcal{M}_1$. Note that we can compute some $K_1 \geq K_0$ such that the following holds for all $T \geq K_1$.
    \begin{equation}\label{eqn:k1}
    1 - \alpha e^{T\beta \varepsilon^2} \leq 1 - \frac{1}{2^T}
    \end{equation}
    We intend to use this in conjunction with the following inequality which holds for all $T \in \mathbb{N}$.
    \begin{equation}\label{eqn:prod-exp}
        \prod_{t = T}^\infty(1 - 2^{-t}) \geq \exp\left( -2^{2-T} \right)
    \end{equation}
    Indeed, we want to choose $M$ such that $M \geq K_1$ and $\prod_{t = T}^\infty(1 - 2^{-t}) \geq 1 - \varepsilon$. Towards this, we note that the following hold.
    \begin{align*}
        \exp\left( -2^{2-T} \right) \geq 1 - \varepsilon \iff & -2^{2-T} \geq \ln(1 -\varepsilon) \\
        \iff & 2^{2-T} \leq \ln \varepsilon \\
        \iff & T - 2 \geq -\log_2(\ln( \varepsilon)) \\
        \iff & T \geq 2 - \log_2(\ln(\varepsilon)) \\ \implies & \prod_{t = T}^\infty(1 - 2^{-t}) \geq 1 - \varepsilon & \text{by \autoref{eqn:prod-exp}}
    \end{align*}
    We therefore set $M \coloneqq \max\{K_1,2-\log_2(\ln(\varepsilon))\}$.
    
    Let us denote by $E_T$ the event satisfying the following for all $T \geq M$.
    \[
    \sum_{t=1}^T r_{X_t}(Y_t) \geq \Ex_{j,\pi^\infty}\left[ \sum_{t=1}^T r_{X_t}(Y_t) \right] - T\varepsilon
    \]
    Note that $\Pr_{j,\pi^\infty}(E_T) \geq \prod_{t=M}^T(1-2^{-t})$ because $M \geq K_1$. Furthermore, since $E_T \subseteq E_{T'}$, for any $T \leq T'$, the following holds (see, e.g.,~\cite[page 756]{bk08}).
    \[
    \Pr_{j,\pi^\infty}\left(\bigcap_{T \geq M} E_T\right) \geq \prod_{t=M}^\infty(1-2^{-t})
    \]
    Then, since $M \geq 2 - \log_2(\ln(\varepsilon))$, the desired result follows.
\end{proof}

\begin{exercise}
    Prove that there exists some $K_1 \in \mathbb{N}$ such that \autoref{eqn:k1} holds for all $0 < \varepsilon < 1$ and all $K_1 \leq T$.
\end{exercise}
    
\begin{exercise}
    Prove that \autoref{eqn:prod-exp} holds for all $T \in \mathbb{N}$.
\end{exercise}

Before we proceed to the proof of the first item of the lemma, we prove an ergodic theorem for unichain policies.
\begin{lemma}\label{lem:ergodic}
    For all $j \in S$ and all unichain policies $\pi^\infty$ the following hold.
    \[
        \{\phi(\pi^\infty)\}_j =
        \liminf_{T\to\infty}\frac{1}{T}
        \Ex_{j,\pi^\infty}\left[\sum_{t=1}^Tr_{X_t}(Y_t)\right] =
        \limsup_{T\to\infty}\frac{1}{T}
        \Ex_{j,\pi^\infty}\left[\sum_{t=1}^Tr_{X_t}(Y_t)\right]
    \]
\end{lemma}
\begin{proof}
    The argument consists in applying Lebesgue's dominated convergence theorem, which gives sufficient conditions for the equivalence between the limit of the expectation of functions and the expectation of their limit. Simply stated, we need the (pointwise) limit of the averages to almost-surely exist and a (finite-expectation) bound on all averages for all possible outcomes. For the second point: recall that the reward function is bounded. Hence, for all outcomes we have bounded averages. It remains to prove the first point.
    
    For finite irreducible Markov chains, the ergodic theorem (see, e.g.,~\cite[Theorem 1.10.2]{norris98}) tells us the following.
    \[
    \Pr\left( 
    \liminf_{T\to\infty}\frac{1}{T}
        \sum_{t=1}^Tr({X_t}) =
        \limsup_{T\to\infty}\frac{1}{T}
        \sum_{t=1}^Tr({X_t})
    \right) = 1
    \]
    Hence, in that case the limit almost-surely exists. To conclude we just observe that the ergodic theorem clearly extends to unichain Markov chains since the function is ``prefix-independent''. That is, ignoring any finite prefix of the sequence of averages before taking the limit yields the same value. Since $\pi^\infty$ induces a unichain Markov chain, the result follows by Lebesgue's dominated convergence theorem.
\end{proof}

We are now ready to prove the folklore result.

\begin{proof}[Proof of \autoref{lem:uni-strats}  \autoref{itm:exp-asurely}]
    Let $\varepsilon_k = 2^{-k}$. We write $E_k$ for the event satisfying the following.
    \[
        \exists M, \forall N \geq M :
        \sum_{t=1}^N r_{X_t}(Y_t) \geq \Ex_{\pi^\infty}\left[\sum_{t=1}^N r_{X_t}(Y_t)\right] - N\varepsilon_{k}
    \]
    Note that $E_k \subseteq E_\ell$ for all $k \leq \ell$. From \autoref{itm:long-convergence} we have that:
    \[
        \Pr_{j,\pi^\infty}(E_k) \geq 1 - \varepsilon_k = 1 - 2^{-k}
    \]
    and therefore:
    \[
        \Pr_{j,\pi^\infty}\left(
        \bigcap_{k \geq 1} E_k
        \right) =
        \lim_{k \to \infty} 1 - 2^{-k} = 1.
    \]
    To conclude the proof we argue that the event $\bigcap_{k \geq 1} E_k$ almost-surely coincides with the required event. Indeed the former is equivalent to the event satisfying the following.
    \begin{align*}
        &\forall k \geq 1, \exists M, \forall N \geq M :
        \sum_{t=1}^N r_{X_t}(Y_t) \geq \Ex_{\pi^\infty}\left[\sum_{t=1}^N r_{X_t}(Y_t)\right] - N\varepsilon_{k}\\
        \iff& \liminf_{T \to \infty} \frac{1}{T}\left(\sum_{t=1}^T r_{X_t}(Y_t) \geq \Ex_{\pi^\infty}\left[\sum_{t=1}^T r_{X_t}(Y_t)\right] \right) \geq 0\\ \iff & \liminf_{T \to \infty} \frac{1}{T} \sum_{t=1}^T r_{X_t}(Y_t) \geq
        \limsup_{T \to \infty} \frac{1}{T}\Ex_{\pi^\infty}\left[\sum_{t=1}^T r_{X_t}(Y_t)\right].
    \end{align*}
    By \autoref{lem:ergodic}, the last event defined above almost-surely coincides with the event:
    \[
    \liminf_{T \to \infty} \frac{1}{T} \sum_{t=1}^T r_{X_t}(Y_t) \geq \phi(\pi^\infty),
    \]
    which thus concludes the proof.
\end{proof}

\subsubsection{Putting everything together}
Let $S_N \coloneqq \sum_{t=0}^{N-1} L_k+O_k$. The proof of \autoref{thm:as-optimal} uses the following intermediate result.
\begin{lemma}
    Let $(\varepsilon_i)_{i \in \mathbb{N}}$ be a sequence of values $0 < \varepsilon_k < \varepsilon_j$ such that for all $j < k$.
    For all $i \in \mathbb{N}$ one can compute $L_i,O_i \in \mathbb{N}$ so that the following holds for all $i \geq 1$.
    \[
        \Pr_R\left(\forall T \in (S_i,S_{i+1}]: \frac{1}{T} \sum_{t=1}^T r_{X_t}(Y_t) \geq \phi - \varepsilon_i \right) = 1 - \varepsilon_i
    \]
\end{lemma}
Let $\varepsilon_i \coloneqq 2^{-i}$ and consider the event $E_i$ defined as satisfying the following.
\[
\forall T \in (S_i,S_{i+1}]: \frac{1}{T} \sum_{t=1}^T r_{X_t}(Y_t) \geq \phi - \varepsilon_i
\]
Note that $D \coloneqq \bigcup_{i \geq 0} \bigcap_{j \geq i} {E_j}$ is the event consisting of samples whose expected limit average is optimal. Hence, to conclude it suffices to argue that $\overline{D}$ --- the complement of $D$ --- has probability measure $0$. The theorem thus follows from the Borel-Cantelli lemma since $\sum_{i \geq 0} \Pr_R(\overline{E_i}) < \infty$ and therefore $\Pr_R(\bigcap_{i \geq 0} \bigcup_{j \geq i} \overline{E_j}) = 0$

We do not provide a proof of the lemma as it is elementary yet tedious. For all $i \geq 0$ one can assume the accumulated reward is $W \cdot S_i$ and then use the bounds from previous sections to choose $L_i$ and $O_i$ to learn and optimize so as to get close enough to the optimal value. Since all convergence bounds we have developed are exponential, finding such values is not a problem.

\section{Model-free discounted-reward optimization}
We turn our attention back to the expected total-discounted reward function. In this case, we will \emph{not try to approximate the unknown MDP}. Instead, our approach consists in approximating the optimal value of states and state-action pairs. That is, we present a learning policy which internally computes increasingly better such approximations. Namely, we will study the classical Q-learning algorithm due to Watkins and Dayan.~\cite{wd92}.

Because of the nature of the total-discounted reward function, no learning policy can really be optimal. However, our policy will be such that --- under some conditions --- its approximations will converge to the optimal ones. Since we avoid constructing an approximation of the unknown MDP, the algorithm we focus on (and similar ones) are sometimes said to lean in a way that is \emph{model free}. 

\subsection{Notation and assumptions}
Henceforth, let us fix an MDP $\mathcal{M} = (S,A,P,r)$ and a discount factor $\alpha \in (0,1)$.
We write $V^* : S \to \mathbb{Q}$ to denote the function that maps every $i \in S$ to its optimal value. Formally, for all $i \in S$, we define:
\[
    V^*(i) \coloneqq \sup_{R \in C(D)}v^\alpha_i(R),
\]
where $C(D)$ still denotes the set of all deterministic stationary policies.
Similarly, we define the \emph{Q values} to refer to the optimal values of state-action pairs. That is, we define $Q^* : S \times A \to \mathbb{Q}$ for all $i \in S$ as follows.
\[
    Q^*(i,a) \coloneqq
    r_i(a) + \alpha \sum_{j \in S} p_{ij}(a) V^*(j)
\]

From the definition of $C(D)$ and the discounted-reward objective, we have the following for all $i \in S$.
\begin{equation}\label{eqn:VQ}
V^*(i) = \max_{a \in A(i)} Q^*(i,a)
\end{equation}

\paragraph{Assumptions.}
We will only assume that the rewards are bounded by $W$, that is $|r_i(a)| \leq W$ for all $i \in S$ and all $a \in A(i)$. Further assumptions will be made explicit in formal statements.

\subsection{The algorithm}
We presently describe the policy $R$ corresponding to the Q-learning algorithm. Though not explicit in all presentations of it, Q-learning is in fact parameterized by a sequence of \emph{learning rates} $\gamma_1, \gamma_2, \dots$ such that $0 \leq \gamma_k < 1$ for all $k \geq 1$.

The policy $R$, at every timestep $t$, keeps an approximation $Q^{(t)}$ of the $Q$ values. We will make no assumption about $Q^{(1)}$ other than that all state-action pairs are mapped to some rational value. For $t \geq 2$, the choice of action $Y_t$ to play from the current state $X_t$ is resolved using: 
\[
    Y_t \coloneqq \argmax_{a \in A(X_t)} Q^{(t)}(X_t,a).
\]
Then,
having received a reward of $r_{X_t}(Y_t)$ before reaching state $X_{t+1}$, the policy internally computes $Q^{(t+1)}$ from $Q^{(t)}$ by changing $Q^{(t)}(X_t,Y_t)$ to the following value.
\[
    (1-\gamma_t) Q^{(t)}(X_t,Y_t) + \gamma_t \left(r_{X_t}(Y_t) + \max_{a \in A(X_{t+1})} \alpha Q^{(t)}(X_{t+1},a)\right)
\]

We now formalize our convergence claim from before. Below, we denote by $N_{ia} : \mathbb{N} \to \mathbb{N}$ the function mapping $n$ to the $n$-th timestep for which $X_t=i$ and $Y_t=a$.
\begin{theorem}\label{thm:qlearning}
    Let $\alpha \in (0,1)$ and $i \in S$. Then for the Q-learning policy $R$ we have $\Pr_{i,R}(\lim_{t \to \infty} Q^{(t)} = Q^*) = 1$ if the following hold with probability $1$ for all $i \in S$ and all $a \in A(i)$.
    \[
    \sum_{n=1}^\infty \gamma_{N_{ia}(n)} = \infty \text{ and }
    \sum_{n=1}^\infty \left(\gamma_{N_{ia}(n)}\right)^2 < \infty
    \]
\end{theorem}

We present the proof from~\cite{wd92} for this claim. The argument consists of two steps. First, we relate the Q values to the optimal values of states in a second MDP we introduce in the next subsection. 

\subsection{The action-replay process}
Consider a sample $X_1 Y_1 X_2 Y_2 \dots$ drawn from the Markov chain induced by the Q-learning policy $R$ and $\mathcal{M}$. We will define what Watkins and Dayan call the \emph{action-replay process} $\mathcal{A}$. It is an MDP $\mathcal{A}= (\widehat{S},A,\widehat{P},\widehat{r})$ where:
\begin{itemize}
    \item $\widehat{S} = S \times \mathbb{N}_{>0} \cup \{\bot\}$;
    \item $\widehat{r}(\langle i, \cdot \rangle,a,\bot) = Q^{(1)}(i,a)$ and $\widehat{r}(\langle i, \cdot \rangle, a, \widehat{j}) = r_i(a)$ for all $i \in S$, all $a \in A(i)$, and all $\widehat{j} \neq \bot$.\footnote{Note that rewards for this MDP depend on the played action and both the source and target states! This is no loss of generality: one can always get back to state-action rewards by adding intermediate states.}
    \item Finally, to define $\widehat{P}(\langle i, t \rangle, a, \widehat{j})$, let $t_1,t_2,\dots,t_k$ be the maximal sequence of timesteps such that $t_\ell = N_{ia}(\ell)$ and $t_\ell < t$ for all $1 \leq \ell \leq k$. Now, we set:
    \[
    \widehat{P}(\langle i, t \rangle, a, \widehat{j}) = \begin{cases}
    \gamma_{t_k} & \text{if } \widehat{j} = \langle X_{t_k + 1}, t_k + 1 \rangle\\
    \gamma_{t_{k-1}}(1 - \gamma_{t_k}) & \text{if } \widehat{j} = \langle X_{t_{k-1} + 1}, t_{k-1} + 1 \rangle\\
    \vdots & \vdots\\
    \gamma_{t_\ell} \prod^k_{m>\ell}(1-\gamma_{t_m}) & \text{if } \widehat{j} = \langle X_{t_{\ell}+1}, t_\ell + 1 \rangle\\
    \vdots & \vdots\\
    \prod^k_{\ell=1}(1-\gamma_{t_{\ell}}) & \text{if } \widehat{j} = \bot\\
    0 & \text{otherwise.}
    \end{cases}
    \]
\end{itemize}

The next result follows from the definition of the action-replay process. We denote by $Q^*_{\mathcal{A}}$ the optimal values of state-action pairs in $\mathcal{A}$.
\begin{lemma}\label{lem:arp-qvalues}
    Let $n$ be a positive integer. Then $Q^{(n)}(i,a) = Q^*_\mathcal{A}(\langle i, n \rangle,a)$ for all $i \in S$ and all $a \in A(i)$.
\end{lemma}
\begin{proof}
    We argue that the claim holds by induction.
    
    For the base case, we observe that if $n = 1$ then $\widehat{P}(\langle i, n \rangle, a, \bot) = 1$ for all $i \in S$ and all $a \in A(i)$. Indeed, the sequence $t_1,\dots,t_k$ is empty in this case. Hence, the product $\prod^k_{\ell=1}(1-\gamma_{t_{\ell}})$ evaluates to $1$.
    Furthermore, by construction, we have that $\widehat{r}(\langle i, n \rangle,a,\bot) = Q^{(1)}(i,a)$. Since $\bot$ has no outgoing transitions, $Q^*_{\mathcal{A}}(i,a) = \widehat{r}(\langle i, n \rangle,a,\bot) = Q^{(1)}(i,a)$. It follows that there is some $n$ for which the claim holds.
    
    For the inductive step, we first observe there are some trivial cases. Consider $\langle i, n \rangle$ and $a \in A(i)$ such that $i \neq X_n$ or $a \neq Y_n$. We have that $Q^{(n+1)}(i,a) = Q^{(n)}(i,a)$. Also, by construction of $\mathcal{A}$, the following hold for all $\widehat{j} \in \widehat{S}$:
    \begin{gather*}
    \widehat{P}(\langle i, n + 1 \rangle, a, \widehat{j}) = \widehat{P}(\langle i, n \rangle, a, \widehat{j}) \text{ and}\\
    \widehat{r}(\langle i, n + 1 \rangle, a, \widehat{j}) = \widehat{r}(\langle i, n \rangle, a, \widehat{j}).
    \end{gather*}
    Hence, $Q^*_\mathcal{A}(\langle i, n + 1 \rangle, a) = Q^*_\mathcal{A}(\langle i, n \rangle, a)$. For such state-action pairs the claim follows from our inductive hypothesis. It remains for us to prove that it also holds when $i = X_n$ and $a = Y_n$. 
    Let us write $\widehat{i} = \langle i, n \rangle$ and $\widehat{j} = \langle X_{n+1}, n \rangle$. In this case, by definition, $Q^{(n+1)}$ is equal to the following.
    \begin{align*}
    &(1-\gamma_n) Q^{(n)}(i,a) + \gamma_n \left(r_{i}(a) + \max_{b} \alpha Q^{(n)}(X_{n+1},b)\right)\\
    = & (1-\gamma_n) Q^*_{\mathcal{A}}(\widehat{i},a) + \gamma_n \left(r_{i}(a) + \max_{b} \alpha Q^*_{\mathcal{A}}(\widehat{j},b)\right) & \text{by IH}\\
    = & (1-\gamma_n) Q^*_{\mathcal{A}}(\widehat{i},a) + \gamma_n \left(r_{i}(a) + \alpha V^*_{\mathcal{A}}(\widehat{j})\right) & \text{\autoref{eqn:VQ}}\\
    = & (1-\gamma_n) Q^*_{\mathcal{A}}(\widehat{i},a) + \gamma_n \left(\widehat{r}(\widehat{i}, a, \widehat{j}) + \alpha V^*_{\mathcal{A}}(\widehat{j})\right) & \text{def. of } \widehat{r}
    \end{align*}
    On the other hand, observe that for all $\widehat{k} \neq \widehat{j}$ the following hold:
    \begin{equation}\label{eqn:n1-to-n}
        \begin{aligned}
        \widehat{P}(\langle i, n+1 \rangle, a, \widehat{k}) = (1 - \gamma_n) \widehat{P}(\widehat{i}, a, \widehat{k})
        \text{ and}\\
        \widehat{r}(\langle i, n+1 \rangle, a, \widehat{k}) = \widehat{r}(\widehat{i}, a, \widehat{k}).
        \end{aligned}
    \end{equation}
    Therefore, we have that $Q^*_\mathcal{A}(\langle i, n+1 \rangle, a)$ is equal to the following.
    \begin{align*}
    &\sum_{\widehat{k} \neq \widehat{j}}
    \widehat{P}(\langle i, n+1\rangle,a,\widehat{k}) \left(\widehat{r}(\langle i, n+1\rangle,a,\widehat{k}) + \alpha  V_\mathcal{A}^*(\widehat{k})\right) +{} \\
    &\widehat{P}(\langle i, n+1\rangle,a,\widehat{j}) \left(\widehat{r}(\langle i, n+1\rangle,a,\widehat{j}) + \alpha  V_\mathcal{A}^*(\widehat{j})\right)\\
    =&(1-\gamma_n)\sum_{\widehat{k} \neq \widehat{j}}
    \widehat{P}(\widehat{i},a,\widehat{k}) \left(\widehat{r}(\widehat{i} ,a,\widehat{k}) + \alpha  V_\mathcal{A}^*(\widehat{k})\right) + {} & \text{\autoref{eqn:n1-to-n}}\\
    & \widehat{P}(\langle i, n+1\rangle,a,\widehat{j}) \left(\widehat{r}(\widehat{i},a,\widehat{j}) + \alpha  V_\mathcal{A}^*(\widehat{j})\right) & \text{def. of } \widehat{r}\\
    =&(1-\gamma_{n})Q^*_\mathcal{A}(\widehat{i},a) +{}  & \text{def. of } Q^*\\
    &\widehat{P}(\langle i, n + 1 \rangle,a,\widehat{j}) \left(\widehat{r}(\widehat{i},a,\widehat{j}) + \alpha  V_\mathcal{A}^*(\widehat{j})\right)\\
    =&(1-\gamma_{n})Q^*_\mathcal{A}(\widehat{i},a) + \gamma_n \left(\widehat{r}(\widehat{i},a,\widehat{j}) + \alpha  V_\mathcal{A}^*(\widehat{j})\right) & \text{def. of } \widehat{P}
    \end{align*}
    The result thus follows by induction.
\end{proof}

We will now consider the probability of staying above a certain ``level'' in the action-replay process. Recall that $\tau(T) \coloneqq \inf\{t \in \mathbb{N} \mid X_t \in T\}$ is the first $T$-hit time and that it is a random variable whose value ranges from $0$ to $+\infty$. We further write $[m]$, for $m \in \mathbb{N}_{>0}$, to denote the set $\{1,2,\dots,m\}$.
\begin{lemma}\label{lem:arp-nodrop}
    Let $\ell,m \in \mathbb{N}_{>0}$ and $\varepsilon \in (0,1)$. Then there exists $n \geq m$ such that
    $\Pr_{\langle i, n\rangle,R}(\tau(S \times [m]) \leq \ell) \leq \varepsilon$ for all $i \in S$ and all policies $R$ in $\mathcal{A}$ if the following holds with probability $1$ for all $i \in S$ and all $a \in A(i)$.
    \[
        \sum^\infty_{k=1} \gamma_{N_{ia}(k)} = \infty
    \]
\end{lemma}
\begin{proof}
    We will argue that $\widehat{P}(\langle i, m \rangle, a, \langle j, m\rangle) \leq \nicefrac{\varepsilon}{\ell}$ where,
    for ease of readability, we focus on an arbitrary $i \in S$ and action $a \in A(i)$. It will be clear how to extend the argument so that $m$ is chosen for the claim to hold uniformly for all such $i$ and $a$. Essentially, the latter is achieved by choosing the most ``restrictive'' state-action pair and the corresponding value of $n$. The claim follows from the above by induction and the Markov property (since the probabilities add up).
    
    Consider some value $n \geq m$.
    By construction of $\widehat{P}$, we have that the probability of hitting $S \times [m]$ in one step, i.e. $\sum_{k=1}^m \widehat{P}(\langle i, m \rangle, a, \langle j, k\rangle)$, is:
    \[
        \left(\prod_{k' > k}^n (1 - \gamma_{N_{ia}(k')})\right)
        \sum_{\ell=1}^{k'-1}
        \left(
        \gamma_{N_{ia}(\ell)}
        \prod_{\ell' > \ell}^{k'-1} (1 - \gamma_{N_{ia}(\ell')})
        \right) \leq \prod_{k' > k}^n (1 - \gamma_{N_{ia}(k')}).
    \]
    Furthermore, note that:
    \begin{equation}\label{eqn:exp-up}
        \prod_{k' > k}^n (1 - \gamma_{N_{ia}(k')}) \leq \exp\left(\sum_{k' > k}^n (1 - \gamma_{N_{ia}(k')})\right)
    \end{equation}
    since $\exp(\cdot)$ is monotonic. Now,
    recall that we have assumed, for all $i \in S$ and all $a \in A(i)$, \(
        \sum^\infty_{k=1} \gamma_{N_{ia}(k)} = \infty
    \) holds with probability $1$.
    Since the $\gamma_t$ are bounded in $[0,1)$ then $\sum_{k' > k}^n (1 - \gamma_{N_{ia}(k')}) = \infty$ with probability $1$ too. In particular, this means that \autoref{eqn:exp-up} tends to $0$ as $n$ goes to infinity. Hence, we can certainly choose $n$ so that \autoref{eqn:exp-up} is bounded above as required.
\end{proof}

\begin{exercise}
    Complete the induction in the proof sketch given above.
\end{exercise}

\subsection{The action-replay process and the unknown MDP}
The second step we take towards proving \autoref{thm:qlearning} consists in establishing a relation between the action-replay process $\mathcal{A} = (\widehat{S},A,\widehat{P},\widehat{r})$ and the unknown dynamics of the MDP. It will be convenient to define a probabilistic transition function induced by $\widehat{P}$ when projected on $S \times A$. Formally, we define $\widehat{P}_t$ as follows.
\[
    \widehat{P}_t : (i,a,j) \mapsto \sum_{s \in \mathbb{N}_{>0}} \widehat{P}(\langle i, t \rangle, a, \langle j, s \rangle)
\]
Note that from any state $\langle i, n\rangle$ in $\mathcal{A}$ and for all actions $a \in A(i)$, all transitions in the process lead to states $\langle j, m \rangle$ such that $m < n$ or to $\bot$. Hence, we also have that the following holds for $\widehat{P}_t$.
\[
\widehat{P}_t(i,a,j) = \sum_{s=1}^{t-1} \widehat{P}(\langle i, t \rangle, a, \langle j, s \rangle)
\]
We also define $\widehat{r}_t$ such that $\widehat{r}_t :(i,a) \mapsto \widehat{r}(\langle i, t\rangle, a, \widehat{j})$ for any $\widehat{j} \in \widehat{S}\setminus \{\bot\}$ for which this is defined. (Note that this is well-defined since $\widehat{r}$ only depends on the target state if the latter is $\bot$.) 
\begin{lemma}\label{lem:arp-sims-mdp}
    For the Q-learning policy $R$ we have: 
    \[
    \Pr_{i,R}\left(\lim_{t \to \infty} \widehat{P}_t = P \text{ and } \lim_{t \to \infty} \widehat{r}_t = r\right) = 1
    \] for all $i \in S$ if the following hold with probability $1$ for all $i \in S$ and all $a \in A(i)$.
    \[
    \sum_{n=1}^\infty \gamma_{N_{ia}(n)} = \infty \text{ and }
    \sum_{n=1}^\infty \left(\gamma_{N_{ia}(n)}\right)^2 < \infty
    \]
\end{lemma}
\noindent We do not present a proof of this claim. Watkins and Dayan give a proof of it using classical \emph{stochastic convergence} results (see, e.g. the Robbins-Monroe algorithm~\cite{kc12}). There seem to exist alternative arguments using ordinary differential equations. However, these techniques go beyond the scope of these notes.

\subsection{Putting everything together}

\autoref{lem:arp-sims-mdp} essentially allows us to argue that we get increasingly better estimates of the desired Q values via the optimal values in the MDP induced by $\widehat{P}_t$. Below, we write $\widehat{Q}_t$ to denote the optimal Q values in the MDP $(S,A,\widehat{P}_t,\widehat{r}_t)$ induced by $\widehat{P}_t$ and $\widehat{r}_t$.
\begin{lemma}\label{lem:qlearn-trick}
    Let $\varepsilon \in (0,1)$ and $R$ be the Q-learning policy. Then there exists $T \in \mathbb{N}_{>0}$ such that $\Pr_{i,R}\left(\forall t \geq T: \dtv{\widehat{Q}_t}{Q^*} > \varepsilon\right) \leq \varepsilon$ for all $i \in S$ if the following hold with probability $1$ for all $i \in S$ and all $a \in A(i)$.
    \[
    \sum_{n=1}^\infty \gamma_{N_{ia}(n)} = \infty \text{ and }
    \sum_{n=1}^\infty \left(\gamma_{N_{ia}(n)}\right)^2 < \infty
    \]
\end{lemma}
\begin{proof}
    We choose $T$ so that $N_{ia}(1) \leq T$ with probability $1$ for all $i \in S$ and $a \in A(i)$. In words, this means that all state-action pairs have almost-surely been witnessed. (This happens with probability $1$ due to our assumptions.) Additionally, we ask that:
    \[
    \Pr_{i,R}\left(\forall t \geq T: \drat{\widehat{P}_t}{P}, \dtv{\widehat{r}_t}{r} \leq \frac{\varepsilon \exp(-2|S|)}{4W(1-\alpha)}\right) \geq 1 - \varepsilon.
    \]
    The latter is guaranteed to hold for a large enough value of $t$ because of \autoref{lem:arp-sims-mdp}. It then follows from \autoref{thm:disc-robust} that, with probability at least $1 - \varepsilon$, the values $\widehat{Q}_t$ and $Q^*$ have a difference of at most $\varepsilon$ as required.
\end{proof}

As a final stepping stone towards the theorem, we prove an analogue of the
above lemma which completes the link between the action-replay process and the
Q values.
\begin{lemma}\label{lem:qlearn-trick2}
    Let $\varepsilon \in (0,1)$ and $R$ be the Q-learning policy. Then there
    exists $T \in \mathbb{N}_{>0}$ such that $\Pr_{i,R}\left(\forall t \geq T:
    \dtv{Q^{(t)}}{Q^*} > \varepsilon\right) \leq \varepsilon$ for all $i
    \in S$ if the following hold with probability $1$ for all $i \in S$ and
    all $a \in A(i)$.
    \[
    \sum_{n=1}^\infty \gamma_{N_{ia}(n)} = \infty \text{ and }
    \sum_{n=1}^\infty \left(\gamma_{N_{ia}(n)}\right)^2 < \infty
    \]
\end{lemma}
Note that the only difference with respect to \autoref{lem:qlearn-trick} is
that the total-variation bound applies to $Q^{(t)}$ and $Q^*$ instead of
$\widehat{Q}_t$ and the latter. Intuitively, policies in $(S,A,\widehat{P}_t,\widehat{r}_t)$ can be seen as playing in the action-replay process for some steps and then following an arbitrary policy. The proof below formalizes this intuition.
\begin{proof}
    To begin, we observe that the total-discounted-rewards function has a
    ``bounded cutoff'' property. Formally, we have that for all $\eta$ there
    exists $T \in \mathbb{N}$ such that:
    \begin{equation}\label{eqn:cutoff}
    v_i^\alpha(R) - \Ex \left[ \sum_{t=1}^T \alpha^{t-1} r_{X_t}(Y_t) \right] \leq \eta
    \end{equation}
    holds for all $i \in S$ and all policies $R$. Since we have assumed
    bounded rewards (with absolute value of at most $W$), the minimal such $T$
    is easy to compute. Indeed, we only need to solve for $T$ in
    $\nicefrac{W}{1-\alpha^{T-1}} \leq \eta$. Henceforth, let $T$ be such that
    \autoref{eqn:cutoff} holds for $\eta = \nicefrac{\varepsilon}{2}$.

    We now further constraint $T$ so that we can apply \autoref{lem:arp-nodrop} for $\nicefrac{\varepsilon}{2}$. Formally, let $T' \geq T$ be such that $\Pr_{\langle i, n\rangle,R}(\tau(S \times [T]) \leq \ell) \leq \nicefrac{\varepsilon}{2}$ for all $n \geq T'$. Observe that, from our assumptions and \autoref{lem:arp-nodrop}, such a $T'$ necessarily exists. Together with the cutoff property proved above and \autoref{lem:qlearn-trick} (adapted for $\nicefrac{\varepsilon}{2}$), this implies the desired result. 
\end{proof}

\autoref{thm:qlearning} is then easy to prove using the above lemma and a similar argument to the one we gave for \autoref{lem:uni-strats}.

\section*{Acknowledgements}
I would like to thank Benny Van Houdt, Ga\"etan Staquet, and Jef Winant for useful feedback on early versions of these notes. 
Additionally, I am grateful to Nils Charlet and Sarah Leyder for having carefully checked the assumptions in the robustness theorems. 
Finally, I also thank Alexander Belooussov and Jolan Depreter for having reported several typos.

\paragraph{Your name could also be here} if you find a typo and report it to the author. Alternatively, you could find shorter, more elegant proofs for any of the claims. Proofs for the claims which have not been proved here are of special interest!

\bibliographystyle{alpha}
\bibliography{refs}

\end{document}